\documentclass{article}

\usepackage{pgf}
\usepackage[linguistics]{forest}

\usepackage{graphicx} 

\usepackage{booktabs} 
\usepackage{multirow} 

\usepackage{amsfonts}
\usepackage{amsthm}

\usepackage{amssymb}

\usepackage{epsfig}

\usepackage{subcaption}
\usepackage{lettrine}
\usepackage{float}

\usepackage{booktabs} 
\usepackage{multirow} 

\usepackage[linesnumbered,ruled,vlined]{algorithm2e}

\usepackage{amsmath}

\DeclareMathOperator*{\argmin}{arg\,min}

\usepackage{adjustbox}

\newtheoremstyle{italic}
  {15pt} 
  {9pt} 
  {\itshape} 
  {} 
  {\bfseries} 
  {.} 
  {.5em} 
  {} 

\theoremstyle{italic}
\newtheorem{definition}{Definition}
\newtheorem{proposition}{Proposition}

\usepackage{flafter}
\usepackage[pdffitwindow=false,
pdfview=FitH,
pdfstartview=FitH,
plainpages=false]{hyperref}
\hypersetup{pdfborder=0 0 0}
\usepackage{bm}

\usepackage{natbib}
\usepackage[english]{babel}  

\renewcommand{\cite}[1]{%
  (\citealp{#1} [\citenum{#1}])%
}

\title{ Era Splitting: Invariant Learning for Decision Trees }

\author{Timothy DeLise}
\author{Timothy DeLise\thanks{Département de mathématiques et de statistique, Université de Montréal, Montreal, QC, Canada. timothy.delise@umontreal.ca \newline Acknowledgment: This research was funded, in part, by the Numerai during a 2023 research internship.}}

\begin{document}

\maketitle

\begin{abstract}

Real-life machine learning problems exhibit distributional shifts in the data from one time to another or from one place to another. This behavior is beyond the scope of the traditional empirical risk minimization paradigm, which assumes i.i.d. distribution of data over time and across locations. The emerging field of out-of-distribution (OOD) generalization addresses this reality with new theory and algorithms which incorporate \textit{environmental}, or \textbf{\textit{era-wise}} information into the algorithms. So far, most research has been focused on linear models and/or neural networks \cite{arjovsky2020invariant, parascandolo2020learning}. In this research we develop two new splitting criteria for decision trees, which allow us to apply ideas from OOD generalization research to decision tree models, namely, gradient boosting decision trees (GBDTs). The new splitting criteria use era-wise information associated with the data to grow tree-based models that are optimal across all disjoint eras in the data, instead of optimal over the entire data set pooled together, which is the default setting. In this paper, two new splitting criteria are defined and analyzed theoretically. Effectiveness is tested on four experiments, ranging from simple, synthetic to complex, real-world applications. In particular we cast the OOD domain-adaptation problem in the context of financial markets, where the new models out-perform state-of-the-art GBDT models on the Numerai data set. The new criteria are incorporated into the Scikit-Learn code base and made freely available online.  

\end{abstract}

\section{Introduction}

Gradient boosting decision trees (GBDTs) \cite{Friedman2001} and descendent algorithms such as those implemented via the Light Gradient Boosting Machine \cite{NIPS2017_6449f44a}, XGBoost \cite{Chen_2016}, and Scikit-learn's HistGradientBoostingRegressor \cite{scikit-learn}, among others, are considered state of the art for many real-world supervised learning problems characterized by medium data size, tabular format, and/or low signal-to-noise ratio \cite{grinsztajn2022treebased, Fieberg2023}. One such data science problem occurs with predicting future returns in the capital markets. When this problem is framed as a standard supervised learning problem, the target variable is assumed to be a continuous score, representing the expected return (alpha) for a particular stock. The input data is the available information about that stock at that particular time. Each stock's input data and associated target for each date are represented as one row in a data set. This problem is generally referred to as predicting \textit{the cross-sectional returns} of the equity market. 

Supervised learning models in their original form are not \textit{aware} of the date information presented in the data. One row of input data is associated with one target. Samples are assumed to be drawn I.I.D. from the training and test sets, per the empirical risk minimization (ERM) principal \cite{NIPS1991_ff4d5fbb}. While in reality $N$ samples are drawn simultaneously from the data set, where $N$ is the number of stocks in the tradable universe. The predictions for all $N$ stocks are created together at the same time. These $N$ data points may not be completely I.I.D. either, as there are measurable positive correlations among stocks in the stock market over any particular time period. Indeed these ideas are the foundation of the capital asset pricing model (CAPM)  \cite{fama_capm} and modern portfolio theory  \cite{Markowitz_1952}, as well as the principles of risk factor models such as the Barra risk model \cite{RePEc:spr:sprchp:978-3-030-91231-4_99}. Thus, the cross-sectional stock prediction problem violates many of the assumptions of the ERM principle implicitly assumed in the typical supervised learning setting. Meanwhile, practitioners will often take a GBDT model off the shelf and implement it under these assumptions.

The discrepancy between the foundations of the ERM principle and the realities of ML application fits into the contemporary literature in the academic field of out-of-distribution generalization (OOD) research \cite{liu2023outofdistribution, arjovsky2020invariant, parascandolo2020learning}. This field of study recognizes that the ERM principle does not hold true for many, if not all, real-life applications. It does not assume that the data distribution of our training and test sets are identical. Instead, the framework assumes a certain distributional shift among disjoint data sets comprising the training data. Models are designed to account for these shifts. In the literature, this concept is often described in terms of \textit{environments} or \textit{domains} \cite{arjovsky2020invariant, peters2015causal, wilds2021}. In the context of this research, an environment is a domain, is an \textit{era}. Different environments define different experimental conditions that can arise when drawing data at different times or from different locations. Only data drawn from the same environment follows the same distribution, but across environments the distribution can change. Data sets in OOD generalization research are composed of data from several different environments.

There is assumed to be some distributional shift in the data-generating process from one environment to another. The purpose of this is to recognize that there may be \textit{spurious} signals that work in a single or even a group of environments, but fail to generalize to out-of-sample (OOS) data. It is shown in \cite{parascandolo2020learning} that these spurious signals can exist when many environments are pooled together. But these spurious signals change or disappear from one environment to another. These spurious signals usually will present themselves as simpler interpretations of the problem, so a naive model will quickly latch onto them. Think of the cow on grass and camel on sand image recognition problem \cite{DBLP:journals/corr/abs-2010-15775}. During training an image classification task for predicting the label of the animal in an image, cows are always displayed on green grass backgrounds and camels always appear on sandy backgrounds. Naive ML models will simply learn to associate the color green with a cow and the color of sand with a camel. Of course, this is a severe over-simplification of the problem. During evaluation, a cow appearing on a sandy background is classified as a camel, and a camel on a grassy background is classified as a cow. A model which latches onto these spurious signals will perform well in-sample, but have severe flaws when evaluated OOS. Fortunately there may also exist \textit{invariant} signals which work in all environments. These signals can often be much more complex than the spurious signals, so naive models are not likely to learn these invariant signals as readily. In this example, the invariant signal is the cow or camel that is actually present in the image. If an ML model truly understands what a camel or cow looks like, then it will have no problem predicting correctly, no matter the background. 

Figure \ref{fig:ood-eras} shows a conceptual example of what eras represent in financial data. Each era spans a period of time. The boxes spanning each era contain descriptors for the macro-economic environment during each time period. The implication is that financial data coming from each time period (era) contains implicit biases due to the specific conditions of each time period. A question mark is used to show that it is impossible to describe the economic conditions in the future, since that time is still unknown. The traditional ML setting, assuming the ERM principle, implies that the data from the in sample period obey all the same laws as data from the OOS period. OOD research recognizes that this assumption is too idealized for real-life situations, where laws affecting outcomes change from one time to another, or from one place to another.

\begin{figure}
    \centering
    \includegraphics[width=0.8\textwidth]{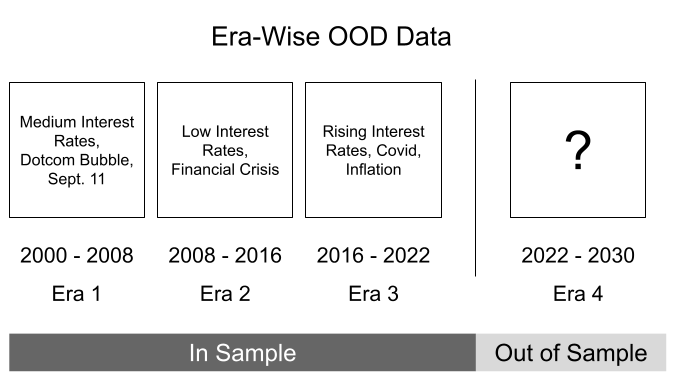}
    \caption[Shifting Financial Environment Over Time]{Financial data experience distribution shifts over time due to changing macro-economic variables and current events. OOD algorithms endeavor to uncover signals which are present in all eras (environments) instead of spurious signals only present in some.}
    \label{fig:ood-eras}
\end{figure}

The purpose of the OOD field of research is to design ML procedures which allow models to ignore the spurious and learn the invariant signals in data, by way of a more accurate representation of the problem. In the work of \cite{arjovsky2020invariant}, which is considered foundational in the field, the authors define their problem as a constrained optimization problem and convert it into a loss function that can be used in gradient descent-style training procedures, such as deep learning with neural networks \cite{Goodfellow-et-al-2016}. The loss function combines the traditional ERM risk term, which minimizes the average error over all environments, and an additional penalty on the magnitude of the gradients of the error in each environment. The gradient norm term is used to measure the optimality of the solution in each environment. Considerable theoretical and data analysis is performed around this idea, and the technique is implemented via neural networks and shows that it works on synthetic data problems designed for this purpose, such as the Colored MNIST problem.

A similarly influential work \cite{parascandolo2020learning} also analyzes the phenomenon of spurious and invariant signals existing simultaneously in data, and how to disentangle the signals with linear models and neural networks. The authors consider the loss surface (the loss function of model parameters), and describe the co-existence of local minima on the surface where some minima correspond to spurious predictors and others correspond to the invariant predictors. The insight here is to realize that spurious minima occur is some environments and not others, but invariant minima occur in all environments in the same location. The authors extend this reasoning and consider the gradients of the loss function with respect to the model parameters. The gradient is constructed as a vector with as many elements as model parameters. The technique is to compute a gradient vector from data in each environment separately. The directions of each gradient element should agree across all environments toward invariant minima but the directions will disagree for spurious signals. In order to realize this understanding practically, the authors alter the traditional gradient descent algorithm by way of the \textit{\textbf{AND-mask}} \cite{parascandolo2020learning}. The AND-mask is a vector the same size as the gradient, containing only ones and zeros (true or false). Elements of the AND-mask corresponding to gradient elements that agree (have the same sign) across all environments are set to 1, and the other elements that disagree (have different signs) are set to zero. The AND-mask is then multiplied element-wise by the gradient of the loss, canceling out elements with disagreement in direction across environments. The \textit{masked} gradient vector is then used in the downstream gradient descent step, propagating the model parameters toward invariant minima in the loss surface and not toward the spurious minima. The authors then show empirically how neural network models can effectively extract invariant predictors from data while ignoring spurious ones. It is important to note that the theoretical insights form the previous two references are limited to the linear settings, while the neural network models function in high dimensions and non-linear signals. 

This current research article investigates the OOD generalization problem in the context of gradient boosted decision trees (GBDTs), and related decision tree algorithms. In particular we focus on \textit{regression trees}, which take a continuous variable as target, as opposed to trees for classification which take a class label as target. This field of research is largely unexplored, with the only reference to the idea known at the time of writing in \cite{liao2024invariant}, which also looks like a promising direction. It is a related idea but distinct from this paper. The researchers define a criterion to identify \textit{invariant} splits from the set of potential splits at each node. The predicted values at each node should be consistent over every era of data, to be invariant. The rule is adapted to a regularization term which is combined with the original splitting criterion. The idea is similar, but less explicit, to directional era splitting, which is explained more in detail later. In this research, we do not employ penalties, but instead replace the original split criterion with our own new criteria.

The rest of the introduction offers a review of decision trees for regression and the GBDT algorithm, highlighting the original splitting criterion used in most off-the-shelf GBDT libraries for regression \cite{XGBoost, LightGBM}. The subsequent sections then describe the novel contributions of this research. The \textbf{\textit{era splitting}} and \textbf{\textit{directional era splitting}} criteria are presented in sections \ref{section:era-splitting} and \ref{section:directional-era-splitting}, which are new contributions of this research. A theoretical breakdown follows in section \ref{sect:theory} which offers key insights and motivations to the design choice of this research. Section \ref{section:methods} describes our experimental methods, specifying the OOD details of each of the 4 experiments. Finally, the paper is concluded with a discussion after presenting results. The results indicate that our new splitting criteria lead to \textbf{better out-of-sample performance} and a \textbf{smaller generalization gap}. In particular, directional era splitting stands out as the best. Links to the open source code base are given in the last section of the paper.

\subsection{Regression Trees}

The basis of the GBDT algorithm for regression are decision trees for regression, also known as \textit{regression trees}. This section follows closely from the chapter, \emph{Regression Trees}, of \cite{Breiman1984}. Regression trees are a practical tool that dates back to the 1960s \cite{Morgan1963ProblemsIT}. It has been akin to linear regression in spirit and underlying theory, but the implementation details are much different.

The problem setup is common in machine learning and statistics, more generally known as \textit{supervised learning}. The data consists of pairs of instances, $(x,y)\in (\mathrm{X}, \mathrm{Y})$, where $X$ is called the \textit{measurement space} and $Y$ is the \textit{target space}. Usually the measurement space is the real numbers of some integer dimension, $d > 0$, and the target space is usually just the real numbers, so $(x,y) \in (\mathbb{R}^d \times \mathbb{R})$. The $x$ are called the independent (or predictor) variables and the $y$ are the dependent (or response) variables. A prediction rule, or predictor, is a real-valued, parameterized function $f(x, \theta)$ on $\mathbb{R}^d$, where $\theta$ represents all the parameters of the predictor. The foundational assumption is made that 
\begin{equation}
    \mathbb{E}[Y|X=x] = f(x, \theta).
\end{equation} \textit{Regression analysis} is the term describing how to construct $f(x, \theta)$ from a training sample $\mathcal{L}$ consisting of $N$ data points $(x_1, y_1), ..., (x_N, y_N)$. Before getting into the specifics of regression trees as a predictor $f$, let's first define terms to help us compute the error of a predictor.

\begin{definition}[Mean Squared Error, \cite{Breiman1984}]
    \label{def:mse}
    The mean squared error $R^*(f)$ of the predictor $f$ is defined as
    \begin{equation}
        R^*(f) = \mathbb{E}[(\mathrm{Y} - f(\mathrm{X}))^2],
    \end{equation}
    where $(\mathrm{X}, \mathrm{Y})$ is an independent sample taken from $\mathcal{L}$.
\end{definition}

Given that the training samples are also assumed to be independently sampled from $\mathcal{L}$, the \textit{re-substitution estimate}, $R(f)$, is the usual method for estimating $R^*(f)$, given by
\begin{equation}
\label{eq:resubstitution-estimate}
R(f) = \frac{1}{N} \sum_{n = 1}^N (y_n - f(x_n))^2.
\end{equation}
In ML parlance this would be called the \textit{training} error or loss, since it is the error calculated with the same data that was used to train the model.

The task of regression is to find parameters of a the predictor function of independent variables that best explains a dependent variable in the least squared sense. If $\theta$ is a finite set of parameters, $\theta = (\theta_1, \theta_2, ... )$, then $\hat{\theta}$ is the parameter value that minimizes the re-substitution error, satisfying 
\begin{equation}
    R(f(x, \hat{\theta})) = \min_\theta R(f(x, \theta)).
\end{equation}
 In linear regression, $f$ takes the functional form $f(x,\theta) = \theta_0 + \theta_1 * x_1 + \theta_2 * x_2 ...$ where $\theta$ is a vector of real-valued numbers to be estimated. In tree regression, $\theta$ represents the tree structure, splitting rules, and parameters that go into defining the tree. 

 Tree predictors function by \textit{\textbf{splitting}} the training data into subsets according to some defined splitting rule. The measurement space $\mathrm{X}$ is partitioned by a series of binary splits so that every value of $x \in \mathrm{X}$ falls into a terminal node. Each terminal node is assigned a constant value that acts as the prediction for data instances that end up in that node. In Figure \ref{fig:tree-example}, an example tree is drawn, taken from \cite{Breiman1984}. Each node $t_i$ for $i \in \{1,2,...,9\}$ is either a parent node or a terminal node. Parent nodes split the data into two child nodes according to the particular split rule for that node. Terminal nodes assign constant values $y(t_i)$ to the nodes. Any data point falls into one of the terminal nodes of the tree structure, which is assigned that node's predicted value. 

\begin{figure}
    \centering
    \pgfkeys{/pgf/inner sep=.6em}
    \begin{forest}
    [$t_1$, circle, draw, name=n1
        [$t_2$, circle, draw, name=n2, s sep=6em
            [$t_4$, draw, name=n4 ]
            [$t_5$, draw, name=n5 ]
        ]
        [$t_3$, circle, draw, name=n3, s sep=6em
            [$t_6$, draw, name=n6 ]
            [$t_7$, circle, draw, name=n7, s sep=6em
                [$t_8$, draw, name=n8 ]
                [$t_9$, draw, name=n9 ]
            ]
        ]
    ]{
        \draw (n1) ++(0em,-2em)++(0em,0pt) node[anchor=north,align=center]{split 1};
        \draw (n2) ++(0em,-2em)++(0em,0pt) node[anchor=north,align=center]{split 2};
        \draw (n3) ++(0em,-2em)++(0em,0pt) node[anchor=north,align=center]{split 3};
        \draw (n4) ++(0em,-1em)++(0em,0pt) node[anchor=north,align=center]{$y(t_4)$};
        \draw (n5) ++(0em,-1em)++(0em,0pt) node[anchor=north,align=center]{$y(t_5)$};
        \draw (n6) ++(0em,-1em)++(0em,0pt) node[anchor=north,align=center]{$y(t_6)$};
        \draw (n7) ++(0em,-2em)++(0em,0pt) node[anchor=north,align=center]{split 4};
        \draw (n8) ++(0em,-1em)++(0em,0pt) node[anchor=north,align=center]{$y(t_8)$};
        \draw (n9) ++(0em,-1em)++(0em,0pt) node[anchor=north,align=center]{$y(t_9)$};
    }
    \end{forest}
    \caption[Example Decision Tree Structure, Figure 8.2 \cite{Breiman1984}]{Figure 8.2 \cite{Breiman1984}, an example tree.}
    \label{fig:tree-example}
\end{figure}
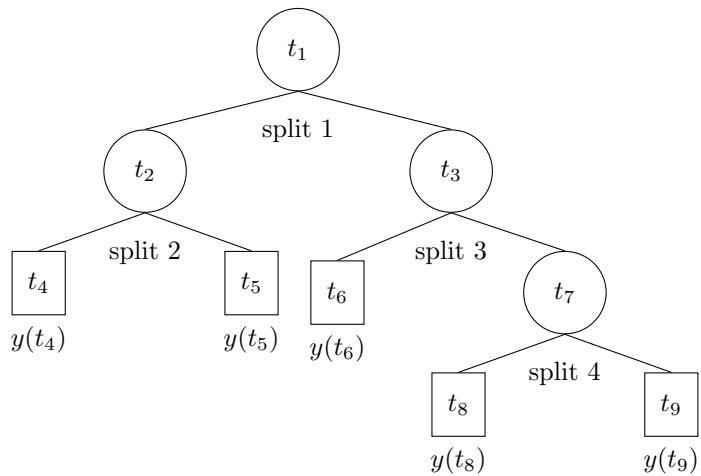

The three necessary elements for building a regression tree from training data $\mathcal{L}$ are:
\begin{enumerate}
    \item A method of selecting a split at each node (a splitting criterion)
    \item A method to determine if a node is terminal
    \item A rule to assign the predicted value $y(t_i)$ at terminal nodes
\end{enumerate}

An important point to understand is that nodes are said to \textit{contain} data points. Indeed, trees are grown according to the training data $\mathcal{L}$. The \textit{root node}, $t_1$ always \textit{contains} all the training data. Split 1 is the rule which partitions all the data into two disjoint subsets, sending the first subset to $t_2$ and the second subset to $t_3$. Each successive node is then split in a similar fashion, according to the data it contains, until a terminal node is reached. The rest of this section focuses on items 1 and 3 from the list above, defining the splitting criterion of each node and the predicted value of terminal nodes. The method for determining when a node is terminal is something a little less precise. In modern libraries, there are several parameters that control when to stop tree growth, such as strict limits on tree depth or total number of leaves. Growing a tree and then reducing the total number of leaves is also called \textit{pruning}. A complete treatment of best practices to limit tree growth is beyond the scope of this text. Interested readers are invited to read the reference \cite{Breiman1984} and review the hyper-parameters of XGBoost \cite{XGBoost} and LightGBM \cite{LightGBM}.

\begin{proposition}[Proposition 8.10, \cite{Breiman1984}]
    \label{def:prop_810}
    The value of $y(t)$ that minimizes $R(d)$ is the average of $y_n$ for all cases $(x_i, y_i)$ falling into $t$; that is, the minimizing $y(t)$ is 
    \begin{equation}
        \bar{y}(t) = \frac{1}{N(t)} \sum_{x_i \in t} y_i
    \end{equation}
    where the sum is over all $y_i$ such that $x_i \in t$ and $N(t)$ is the total number of cases in $t$.
\end{proposition}

Now we take the predicted value in any node $t$ to be $\bar{y}(t)$ and set
\begin{equation}
    R(T) = \frac{1}{N} \sum_{t \in \tilde{T}} \sum_{x_i \in t} (y_i - \bar{y}(t))^2,
\end{equation}
where $\tilde{T}$ is the set of all terminal nodes, and $ \sum_{x_i \in t} (y_i - \bar{y}(t))^2$ is the \textit{within node sum of squares}. Notice the relation between $R(T)$ and $R(f)$. Given a set $\mathcal{S}$ of potential splits in any node $t \in \tilde{T}$,

\begin{definition}[Definition 8.13, \cite{Breiman1984}]
    \label{def:best_split}
    The best split $s^*$ of $t$ is that split in $\mathcal{S}$ which most decreases $R(T)$.
\end{definition}

The format of this definition that is more common is to consider that any $s \in \mathcal{S}$ of $t$ splits the data into $t_L$ (left child node) and $t_R$ (right child node), and let
\begin{equation}
    \label{eq:split-criterion-1}
    \Delta R(s, t) = R(t) - R(t_L) - R(t_R),
\end{equation}
then the best split $s^*$ is the one which satisfies
\begin{equation}
    \Delta R(s^*, t) = \max_{s \in \mathcal{S}} \Delta R(s, t).
\end{equation}

The mean-squared error is analogous to the \textit{impurity measure} for classification trees. The best split is the one which results in the greatest reduction in impurity of the child nodes. In practice, all split rules are defined by one of the independent variables, $x_i$, and some real value $v$, such that all data points with $x_i <= v$ are sent to $t_L$ and the other data points go to $t_R$. Define such a split of node $t$ as $s(x_i, v)$. Thus, the set $\mathcal{S}$ contains all possible splits $s(x_i, v)$ $\forall$ $i \in N(t)$ and $\forall v \in \mathbb{R}$ and each node. Since data sets are finite, the search for the best $v \in \mathbb{R}$ comes down to searching over the number of distinct values of $x_i$ in $\mathcal{L}$.

This subsection described regression trees and how they fit into the field of regression analysis as a type of least squares approximation in function space. We've also reviewed how trees are grown, how tree nodes are split, and how predicted values are derived. In the next subsection we will put these regression trees into context with a review GBDTs for regression.

\subsection{GBDTs for Regression}

The previous section described how to grow a tree on a training data set, by splitting nodes at each step, successively reducing the error of the tree model predictor. GBDTs grow whole trees at each step, resulting in predictors which are a linear combination of many trees. Moreover, each tree is not regressed on the dependent variables $y_i$ themselves, but on the \textit{gradients} of the dependent variables with respect to each data point. This section provides a minimum foundation of GBDTs, following closely from \cite{Friedman2001}, concluding with the definition of the \textit{original split criterion} \cite{XGBoost} that is the standard way to score splits in GBDTs for regression.

This section continues with the notation developed in the previous subsection on regression trees. The general machine learning problem, also called the "predictive learning" problem, is described in \cite{Friedman_2001} as optimal function estimation. The goal is to find an estimate $\hat{F}(x)$ for the function $F^*(x)$, which maps from $\mathrm{X}$ to $\mathrm{Y}$, that minimizes the expected value of some loss function over the joint distribution of all $(x,y) \in (\mathrm{X}, \mathbb{R})$ 
\begin{equation}
    F^* = \min_F \mathbb{E}_{x,y} R^*(F) = \min_F \mathbb{E}_x \left[\mathbb{E}_y \left[ R^*(F) \right]|x \right].
\end{equation}
with $R^*(F)$ the mean-squared error loss function, as defined in definition \ref{def:mse}. GBDTs belong to a class of functions called \textit{additive expansions}, which are a parameterized class of functions $F(x, \theta)$, with $\theta$ a generic list of parameters, having the form 
\begin{equation}
    \label{eq:additive-expaansion}
    F(x,\{\beta_m, a_m\}_1^M) = \sum_{m=1}^M \beta_m h(x, a_m),
\end{equation}
with $\{\beta_m, a_m\}_1^M$ the set of parameters and $h(x, a_m)$ a generic function. In practice $h$ will be a regression tree. This parametric formulation changes the function optimization problem to that of parameter optimization. It is described in \cite{Friedman2001} how the optimal value of the parameters is often attained by starting with an initial guess, and taking successive \textit{steps} or \textit{boosts} based on the sequence of the previous steps toward the optimal value. 

Steepest-descent, also known as gradient descent, is one such method. The gradient for element $j$ of the gradient vector at the $m^{\text{th}}$ step is computed as 
\begin{equation}
    \{g_{j,m}\} = \left\{ \left[ \frac{\partial \mathbb{E}_{x,y} R(F(x,\theta)) }{\partial \theta_j}\right]_{\theta = \theta_{m-1}} \right\},
\end{equation}
where $\theta_{m-1}$ are the parameters resulting from the previous step. That parameter is computed as
\begin{equation}
\theta_{m-1} = \sum_{i=0}^{m-1} p_i,
\end{equation}
where $p_i$ are the results of the boosts at each step. These are defined by
\begin{equation}
    p_i = -\rho g_m
\end{equation}
where $g_m$ is the vector containing all the $\{g_{j,m}\}$, and $\rho$ is known as the \textit{learning rate}. This is one technique for numerical optimization in parameter space. An interesting pivot is to consider optimization of a non-parameterized function $F(\mathrm{x})$ in the same way, where the data, $\mathrm{x}$, is taken to be the "parameter". The goal is to minimize
\begin{equation}
    \min_F \mathbb{E}_y \left[ L(y,F(\mathrm{x}))|\mathrm{x} \right]
\end{equation}
where $L$ is some loss function. Following the numerical optimization with additive expansion paradigm, take the solution to be of the form
\begin{equation}
    F^*(\mathrm{x}) = \sum_{m=0}^M f_m(\mathrm{x}),
\end{equation}
with $f_0(\mathrm{x})$ the initial guess and $f_i(\mathrm{x})$ the subsequent boosting rounds for $i \in \{1, 2, ..., M\}$. For steepest decent we then have
\begin{equation}
    f_m(\mathrm{x}) = -\rho g_m(\mathrm{x}),
\end{equation}
with
\begin{equation}
    g_m(\mathrm{x}) = \left\{ \left[ \frac{\partial \mathbb{E}_y \left[ L(y,F(\mathrm{x})) |\mathrm{x} \right] }{\partial \mathrm{x}}\right]_{F(\mathrm{x}) = F(\mathrm{x})_{m-1}} \right\},
\end{equation}
where
\begin{equation}
    F_{m-1} = \sum_{i=0}^{m-1} f_i(\mathrm{x}).
\end{equation}
This ideal approach is non-parametric and assumes continuous data. In practice data is finite. With sufficient regularity integration and differentiation can be exchanged, resulting the gradient estimate
\begin{equation}
    g_m(\mathrm{x}) = \mathbb{E}_y \left[ \frac{\partial  \left[ L(y,F(\mathrm{x})) |\mathrm{x} \right] }{\partial \mathrm{x}}\right]_{F(\mathrm{x}) = F(\mathrm{x})_{m-1}}.
\end{equation}
And so 
\begin{equation}
    \label{eq:grad-descent}
    F_m(\mathrm{x}) = F_{m-1} - \rho g_{m-1}(\mathrm{x}).
\end{equation}

The final step here is a switch back to an assumption of a parameterized additive expansion, as in equation (\ref{eq:additive-expaansion}). This enforces a certain smoothness over the data, allowing to interpolate values for $x$ lying in between the data points $\{x_i\}_1^N$. The function $h$ from equation (\ref{eq:additive-expaansion}) is known as the \textit{base learner} or \textit{weak learner}. In regression, the base learner is a regression tree, which can define an output in $\mathbb{R}$ for any input in the input space $\mathrm{X}$. The optimization problem becomes a parameter estimation problem again
\begin{equation}
    \{\beta_m, a_m\}_1^M = \min_{\hat{\beta}_m, \hat{a}_m} \sum_1^N L \left( y_i, \sum_{m=1}^M \hat{\beta}_m h(x_i, \hat{a}_m) \right).
\end{equation}
The reference \cite{Friedman_2001} then introduces a \textit{greedy-stagewise} approach to this minimization, for $m \in \{1,2,...,M\}$, with
\begin{equation}
    (\beta_m, a_m ) = \min_{\beta, a} \sum_1^N L \left( y_i, F_{m-1}(x_i) + \beta h(x_i, a) \right)
\end{equation}
where
\begin{equation}
    \label{eq:boosting}
    F_m(\mathrm{x}) = F_{m-1}(\mathrm{x}) + \beta h(\mathrm{x}, a).
\end{equation}
This is what is called \textit{boosting} in the ML literature. Notice the similarity between equations (\ref{eq:grad-descent}) and (\ref{eq:boosting}). The function $\beta h(\mathrm{x}, a)$ is analogous to the steepest descent step toward $F^*(\mathrm{x})$ under the additive expansion assumption of equation (\ref{eq:additive-expaansion}). However, the the gradients $g_m(\mathrm{x})$ can be computed at the data points $x_i$, while $h(\mathrm{x}, a)$ is a parameterized function that we wish to give a best estimate of $g_m(\mathrm{x})$ for all $\mathrm{x}$. In order to do that, the optimal parameters $a_m$ for the regression tree $h$ are obtained by
\begin{equation}
    a_m = \argmin_{a, \beta} \sum_{i=0}^N [ -g_m(x_i) - \beta h(x_i, a) ]^2.
\end{equation}
The \textit{gradient boosting} model is complete, by defining the update step as follows, \begin{equation}
    F_m(\mathrm{x}) = F_{m-1}(\mathrm{x}) + \rho_m h(\mathrm{x}, a_{m}).
\end{equation}

In the preceding derivation from \cite{Friedman2001} leaves out a step from the reference, which optimizes the learning rate(s) $\rho$ for each update. This was left out because in modern ML libraries, the learning rate is usually chosen as a hyper-parameter of the model and kept constant throughout training. Most likely, this preference over time has had to do with the classic \textit{bias-variance trade off} \cite{hastie01statisticallearning}, where fitting training data perfectly can be detrimental to model generalization. Usually, optimal parameters are tuned using a hold-out or \textit{test} set of data or by cross-validation \cite{Breiman1984}.

Finally, the first algorithm presented is the gradient boosting algorithm, without tuning the learning rate, in its general form. $L$ is any differentiable loss function, and $\rho$ is the learning rate.

\begin{algorithm}
\caption{Algorithm 1: Gradient Boost \cite{Friedman_2001}}
\label{alg:gradient-boost}
$F_0(\mathrm{x}) = \argmin_\rho \sum_{i=1}^N L(y_i, \rho)$

\For{$m=1$ to $M$}{
    \For{$i=1$ to $N$}{
        $g_i =  \left[ \frac{\partial  \left[ L(y_i,F(x_i)) |x_i \right] }{\partial x_i}\right]_{F(\mathrm{x}) = F(\mathrm{x})_{m-1}}$
    }
    $a_m = \argmin_{a, \beta} \sum_{i=1}^N [g_i - \beta h(x_i, a)]^2$

    $F_m(\mathrm{x}) = F_{m-1} + \rho h( \mathrm{x}, a_m )$

}
\end{algorithm}

In the case of least squares regression, when the loss function is $L(y, F(\mathrm{x})) = (y - F(\mathrm{x}))^2 / 2$, then the gradients at each step are just $g_i = y_i - F_{m-1}(\mathrm{x})$. This finishes our derivation of the gradient boosting algorithm for regression, in the following algorithm.

\begin{algorithm}
\caption{Algorithm 2: LS Boost \cite{Friedman_2001}}
\label{alg:ls-boost}
$F_0(\mathrm{x}) = \argmin_\rho \sum_{i=1}^N L(y_i, \rho)$

\For{$m=1$ to $M$}{
    \For{$i=1$ to $N$}{
        $g_i =  y_i - F_{m-1}(\mathrm{x})$
    }
    $a_m = \argmin_{a, \beta} \sum_{i=1}^N [g_i - \beta h(x_i, a)]^2$
    
    $F_m(\mathrm{x}) = F_{m-1} + \rho h( \mathrm{x}, a_m )$

}
\end{algorithm}

When the base learner $h$ is a regression tree, then algorithm \ref{alg:ls-boost} defines the GBDT for regression algorithm. In this case, notice that line 4 of the algorithm is a least-squares fit of a regression tree via the re-substitution estimate, $R(f)$, of equation (\ref{eq:resubstitution-estimate}).

\subsection{Splitting}

Optimizing a regression tree $f$ is equivalent to finding the optimal split points of the input data $x_i$ at each node until a terminal node is reached. Each split point is defined by a single feature and a single value of that feature's data. All the data with that feature less than or equal to the split value gets partitioned into the left child node, and the data points with the feature greater than the split value go to the right child node. In successive steps the child nodes become parent nodes, and the data in those nodes is split in the same manner as the root node. Nodes are split like this until terminal nodes are reached. The function that determines the optimal split point at each node is called the \textit{splitting criterion}, which measures the reduction in \textit{impurity}. The split criterion was introduced in equation (\ref{eq:split-criterion-1}). The minimum value of $\Delta R(s,t)$ over all splits $s$ in the set of all splits $\mathcal{S}$ is the split which most reduces most the average error of the tree prediction, $R(T)$, as per definition \ref{def:best_split} above. Minimizing $\Delta R(s,t)$ is equivalent to maximizing \textit{the original split criterion}, given in the next definition.

\begin{definition}[Original Split Criterion, \cite{XGBoost}]
\label{def:original_split_criterion}
Define the original split criterion as 
    \begin{equation}
    \label{original_split_criterion}
    \mathcal{L}_{\text{split}} = \frac{1}{2} \left[ \frac{ \left(\sum_{i \in I_L} g_i\right)^2}{\sum_{i \in I_L} h_i + \lambda } + \frac{ \left(\sum_{i \in I_R} g_i\right)^2}{\sum_{i \in I_R} h_i + \lambda } - \frac{ \left(\sum_{i \in I} g_i\right)^2}{\sum_{i \in I} h_i + \lambda } \right],
    \end{equation}
where $I$, $I_L$, and $I_R$ are data set identifiers corresponding to the data of the parent node, the left child node, and the right child node respectively. The $g_i$ are the gradients as defined in algorithm \ref{alg:ls-boost}, above, $h_i$ are the \textit{hessians} (which are constant for regression), and $\lambda$ is the L2-regularization term.
    
\end{definition}

With the original split criterion defined, we have outlined the state-of-the-art in GBDTs, which we will take as baseline. The original contributions of this paper will be 2 new alternative splitting criteria inspired by recent trends in the OOD generalization research. For this reason we need to first define eras more concisely.

\subsection{Eras (Environments)}

This research employs existing conventions to define what are called \textit{eras}, also known as  \textit{environments} \cite{peters2015causal,arjovsky2017wasserstein} and sometimes \textit{domains} as in \cite{wilds2021}. This paragraph paraphrases from \cite{peters2015causal}, which is the standard setup in related works \cite{arjovsky2020invariant} and \cite{parascandolo2020learning}. Firstly, assume there are different experimental conditions belonging to a the set of all experimental conditions, $e \in \mathcal{E}$, and we take an i.i.d. sample $(\mathrm{X}^e, \mathrm{Y}^e) \in (\mathbb{R}^d, \mathbb{R})$ from each environment, where $\mathrm{X}^e$ and $\mathrm{Y}^e$ are the independent and dependent variables respectively. In particular, the different distributions of $X^e$ in the environments are unknown and not precisely controlled. 

\begin{definition}[Era (Environment) \cite{peters2015causal}]
    \label{def:era}
    An \textbf{era (environment)} refers to a particular \textit{experimental conditions} $e$ from the set of all possible conditions $\mathcal{E}$.
\end{definition}

For this research, it is assumed that every data point originates from some environment, and our entire data set is the union of the data from all the environments. Practically speaking, this comes down to assigning an environmental identifier $j$ to each data point. The identifier is an integer referring to the index of the environment the data comes from. In the following text it is usually assumed that we have data from $M$ eras, and so the training data $(X,Y)$ is the union of data from all $M$ environments  
\begin{equation}
    (X,Y) = \bigcup_{j=1}^M (X^j,Y^j).
\end{equation}

In the original setting, GBDTs pool all the training data together when computing the splitting criterion. If the training data inherently comes from separate environments (eras), as defined in the OOD literature \cite{peters2015causal, arjovsky2020invariant}, what is lost by pooling all the training data together to find the best split? Would it be better to employ a splitting criterion that treats data from each environment (era) separately? This research endeavors to show that the answer to this question is, "yes". 

\subsection{New Contributions and Outline}

A new splitting criterion called \textit{era splitting} is proposed in section \ref{section:era-splitting}. It combines the impurity reduction calculation from each era separately into one final new criterion, which is a smooth average of the impurity reduction of the era-wise splits. It captures the essence of \textit{invariance} from the IRM paradigm \cite{arjovsky2020invariant} applied to the original split criterion. Another new splitting criterion, called \textit{directional era splitting}, is presented in section \ref{section:directional-era-splitting}, which is designed after the AND-mask \cite{parascandolo2020learning}, and looks for splits which maximize the agreement in the \textit{direction} of the predictions implied by the split for each era (environment). In section \ref{sect:theory}, some theoretical concepts are discussed, highlighting the motivation for the new splitting criteria. Section \ref{section:methods} then describes the experimental design used to validate our new splitting criteria. The experiments include a one-dimensional toy model aimed at verifying the objectives in a simple setting, a high-dimensional synthetic dataset proposed for the same purpose in \cite{parascandolo2020learning}, and two real-world empirical studies. The first is the Camelyon17 dataset \cite{bandi2018detection} breast cancer detection domain generalization problem, the second is the financial stock market cross-sectional return problem provided by the Numerai data science tournament. In section \ref{results}, the results from the experiments are presented. Finally section \ref{section:discussion} provides interpretation of the results and discussion about future work.

\section{Era Splitting}
\label{section:era-splitting}

 In the context of the original splitting criterion in equation (\ref{original_split_criterion}), the data set definitions $I$, $I_L$, and $I_R$ need to be refined when dealing with environments. With era splitting, each training data point comes from one of $M$ distinct eras. Thus there exists data sets $I^j$, $I_L^j$, and $I_R^j$, referring to data coming from the $j^{th}$ era, such that $I^j \in I$, $I_L^j \in I_L$, and $I_R^j \in I_R$ for $j \in \{ 1, 2, ..., M \}$. Using this convention, a new era-wise splitting criterion, which computes the information gain of a split for the $j^{th}$ era of data is given by

\begin{equation}
\label{eq:era-loss}
    \mathcal{L}_{\text{split}}^j = \frac{1}{2} \left[ \frac{ \left(\sum_{i \in I_L^j} g_i\right)^2}{\sum_{i \in I_L^j} h_i + \lambda } + \frac{ \left(\sum_{i \in I_R^j} g_i\right)^2}{\sum_{i \in I_R^j} h_i + \lambda } - \frac{ \left(\sum_{i \in I^j} g_i\right)^2}{\sum_{i \in I^j} h_i + \lambda } \right],
\end{equation}
which is just the original split criterion applied to the data from one particular era. The era splitting criterion in its basic form is written as a mean average of the era-wise splitting criterion for each era of data.

\begin{equation}
\label{average_era_split_criterion}
\mathcal{L}_{\text{era split}} = \frac{1}{M} \sum_{j=1}^M \mathcal{L}_{\text{split}}^j
\end{equation}

\begin{figure}
    \centering
    \includegraphics[width=0.48\textwidth]{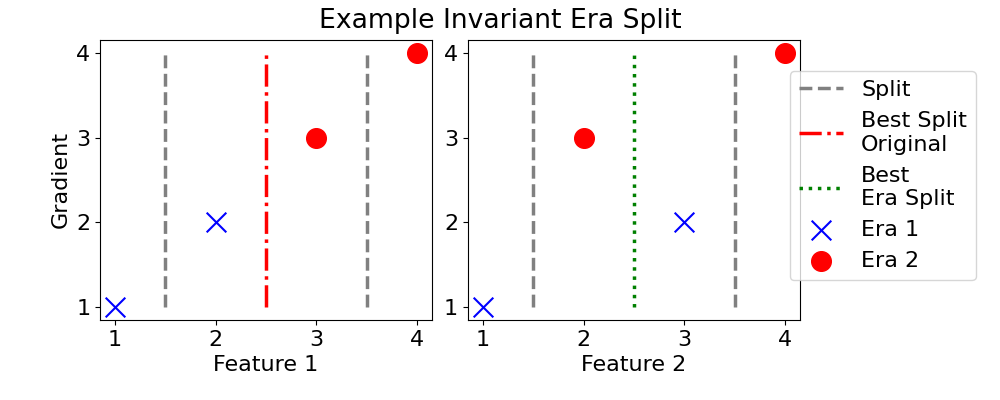}
    \includegraphics[width=0.48\textwidth]{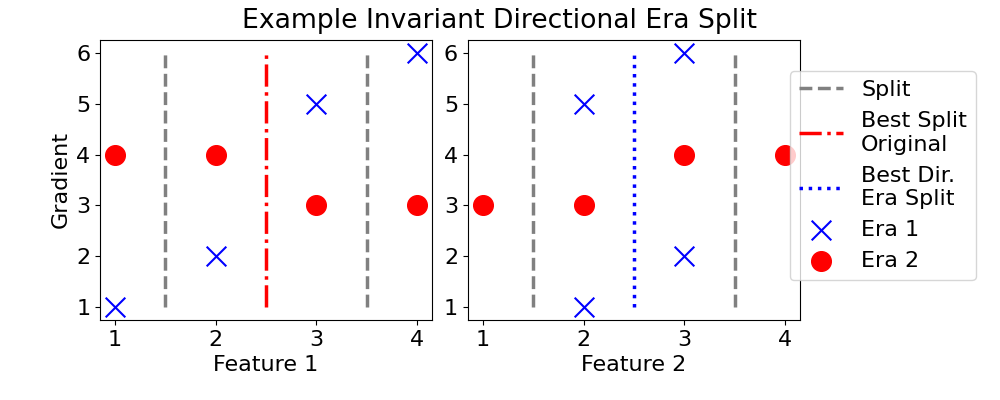}
    \caption[Example Splitting Problems with Eras (Environments)]{Degenerate split decisions induced by the traditional splitting criterion (Eq. \ref{original_split_criterion}) when viewed from an OOD setting. Each plot displays example data with 2 features from two eras (environments). The y-axis plots the gradients corresponding to each data point. \textbf{On the left}, the original splitting criterion chooses a split which doesn't improve impurity in any era. The era splitting criterion (Eq. \ref{era_split_criterion}) chooses a split that improves impurity in both eras. \textbf{On the right} the original criterion chooses a split which results in conflicting directions, while directional era splitting (Eq. \ref{eq:dir-era-split}) chooses a split resulting in consistent directions in each era.  }
    \label{fig:example-degen-era-split}
\end{figure}

This criterion defined in equation (\ref{average_era_split_criterion}) will have the highest score for the split which provides the greatest average decrease in impurity over all of the eras of data. Let us notice that this new criterion recovers equation \ref{original_split_criterion} in the case where there is only one era of data. Therefore the era splitting criterion can be viewed as a generalization of the original splitting criterion. This criterion will be less-likely to favor splits which work really well in some eras but not in others. Splits are not considered which are undefined in any era. Potential split points must have data to the left and right of them to be considered as optimal. The reason for this choice is illustrated in figure \ref{fig:example-degen-era-split}. The plot on the left shows a data configuration where the original splitting criterion will choose a split which doesn't induce an impurity decrease in any era of data by itself. On the other hand, the criterion in equation (\ref{average_era_split_criterion}) chooses a split reducing impurity in both eras simultaneously.

A further generalization is made by way of a smooth max function, called the Boltzmann operator \cite{pmlr-v70-asadi17a}. The Boltzmann operator operates on an array of real numbers, and has one parameter, called $\alpha$. When $\alpha = 0$, the Boltzmann operator recovers the arithmetic mean, when $\alpha = -\infty$ it recovers the $\min$, and when $\alpha = \infty$ it recovers the $\max$. It is known as a \textit{smooth max (min)} function.

\begin{definition}[The Boltzmann Operator]
Let $x_i$ for $i \in \{1, ..., n\}$ be $n$ distinct real numbers, and $alpha \in [ -\infty, \infty ]$, then define the Boltzmann Operator $\mathcal{B}_\alpha$ as 
    \label{def:boltzmann}
    \begin{equation}
        \label{boltzmann_operator}
        \mathcal{B}_\alpha (x_1, ..., x_n) = \frac{\sum_{i=1}^n x_i \exp{ ( \alpha x_i )} }{\sum_{i=1}^n \exp{ ( \alpha x_i )}  }.
    \end{equation}
\end{definition}

\begin{proposition}[Limit of the Boltzmann Operator]
    \label{prop:boltz_limit}
    \begin{equation}
        \lim_{\alpha \to \infty} \mathcal{B}_\alpha(x) = \max (x)
    \end{equation}
\end{proposition}

\begin{proof}
    Let $$\hat{x} = \max (x)$$ and factor the $\hat{x}$ term from both the numerator and denominator of $\mathcal{B}_\alpha(x)$
    \begin{equation}
        \mathcal{B}_\alpha(x) = \frac{ \hat{x} e^{\alpha \hat{x}} + \sum_{x_i \neq \hat{x}} x_i e^{\alpha x_i }}{ e^{\alpha \hat{x}} + \sum_{x_i \neq \hat{x}} e^{\alpha x_i }}.
    \end{equation}
    Now, divide numerator and denominator by $e^{\alpha \hat{x}}$, 
    \begin{equation}
        \mathcal{B}_\alpha(x) = \frac{ \hat{x}  + \sum_{x_i \neq \hat{x}} x_i e^{-\alpha(\hat{x} -  x_i) }}{ 1 + \sum_{x_i \neq \hat{x}} e^{-\alpha(\hat{x} -  x_i)}}.
    \end{equation}
    The term $\hat{x} -  x_i > 0$ for all $ x_i : x_i \neq \hat{x}$, which means that 
    \begin{equation}
        \lim_{\alpha \to \infty} e^{-\alpha(\hat{x} -  x_i)} = 0, \forall x_i : x_i \neq \hat{x}.
    \end{equation}
    Thus, the limit of $\mathcal{B}_\alpha(x)$ is found by taking the limit of numerator and denominator, noticing that the sums of exponential functions all go to zero, yielding
    \begin{equation}
        \lim_{\alpha \to \infty} \mathcal{B}_\alpha(x) = \frac{ \hat{x}  }{ 1 } = \hat{x}.
    \end{equation}
\end{proof}

Proposition \ref{prop:boltz_limit} for the limit and its proof can be given similarly to prove that $\lim_{\alpha \to -\infty} \mathcal{B}_\alpha(x) = \min(x)$. The final version of the era splitting criterion is defined next, which is understood as the smooth maximum (minimum) of the era-wise impurity reduction over all of the eras, and is a generalization of the original splitting criterion in Eq \ref{original_split_criterion}. 

\begin{definition}[Era Split Criterion]
    The \textit{era split criterion} $\mathcal{L}_{\text{era split}}^\alpha$ is defined by 
    \begin{equation}
    \label{era_split_criterion}
    \mathcal{L}_{\text{era split}}^\alpha = \mathcal{B}_\alpha \left( \mathcal{L}_{\text{split}}^1, ..., \mathcal{L}_{\text{split}}^M \right).
    \end{equation}
\end{definition}

The default value of $\alpha$ is set to zero, which recovers equation \ref{average_era_split_criterion}. Conceptually, varying $\alpha$ toward $-\infty$ will lead the criterion to prefer splits which increase the minimum information gain over all of the eras. This setting prefers splits that work in all eras. Conversely, varying $\alpha$ toward positive $\infty$ will cause the criterion to favor splits that improve the best single-era performance. In the prior case splits are preferred that have a more uniform affect on impurity decrease over all the eras. In the later case splits are preferred that improve our best performance in any era. In the context of learning an invariant predictor, negative values for $\alpha$ are preferred.

\section{Directional Era Splitting}
\label{section:directional-era-splitting}

\begin{figure}
    \centering
    \includegraphics[width=0.8\linewidth]{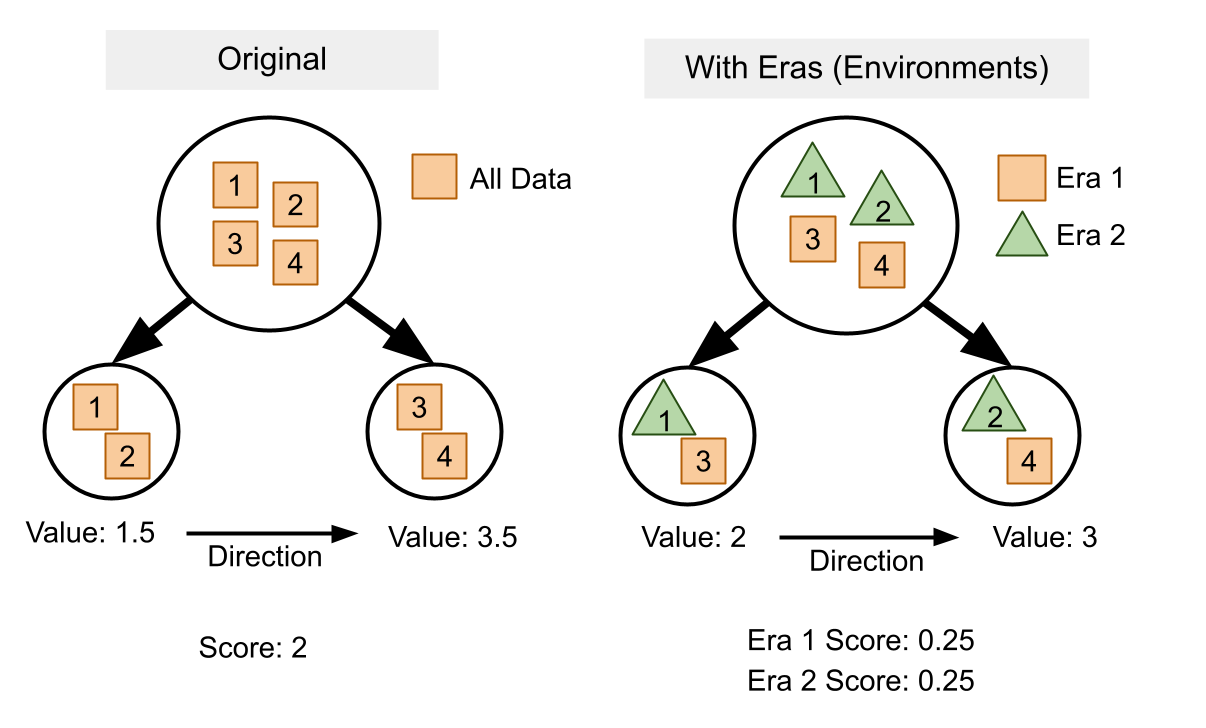}
    \caption[Illustrating Splitting over Eras and Directions of Splits]{ A schematic example of splitting data at each tree node. The target values are stored inside each data point. The original setting pools all the data together. Era splitting computes split scores on a per era basis. The value of the child nodes, the \textit{directions} of the splits and the original and era split criteria scores are indicated. Notice era splitting does not choose the same split as the original, since it wouldn't improve impurity in any era.  }
    \label{fig:splitting-example}
\end{figure}

There is a subtle ambiguity that arises with equation \ref{era_split_criterion}, which requires the introduction of the concept of the \textit{direction} implied by a split. A split implies a direction via the predicted values of the child nodes. Remember from proposition \ref{def:prop_810} that the predicted value at any node is the mean of all the dependent variables $y$ contained in the node. See figure \ref{fig:splitting-example}. If the value of the left child node is larger than the value of the right child node, then the direction of this split is too the left. If the value of the right child node is larger, then the direction is to the right. Let the value of the left child node be $v_l$, and the value of the right child node be $v_r$, then without loss of generality, the direction of the split, $d_\text{split}$, is defined as the sign of the difference between the left child node value and the right child node value

\begin{equation}
\label{split_direction}
d_\text{split} = \text{sign}(v_l - v_r).
\end{equation}

 In the original setting where all the training data is pooled together, there can be only one direction implied over the entire data set. The traditional splitting criterion (eq. \ref{original_split_criterion}) does not consider the direction of the split, since it is superfluous. It doesn't matter whether the direction implied by the split is to the left or to the right, since there is only one era, just as long as the impurity is reduced.
In era splitting, the split criterion is computed for each era separately, and then summarized with the Boltzmann operator. A problem that arises here is that there could be splits that imply conflicting directions from one era to the next. This is akin to single variable linear regression models having coefficients with opposite signs. We can describe the direction of a split in era $j$ by adapting equation \ref{split_direction} with era-wise descriptors, as follows.

\begin{equation}
d_\text{era split}^j = \text{sign}(v_l^j - v_r^j)
\end{equation}

The directions implied by the split for each era of data results in an array of length $M$, $\{ d_\text{era split}^1, d_\text{era split}^2, ..., d_\text{era split}^M  \}$. 
\begin{definition}
    \label{def:dir-era-split}
        Directional era splitting, $\mathcal{D}_\text{era split}$, is the average agreement in the direction of a split over all $M$ eras.
    
    \begin{equation}
    \label{eq:dir-era-split}
    \mathcal{D}_\text{era split} = \frac{1}{M} \left| \sum_{j = 1}^M d_\text{era split}^j \right|
    \end{equation}
\end{definition}

Each $d_\text{era split}^j$ will take a value of either $1$ or $-1$, depending on the direction of the split for that era. The sum is then bounded on the interval $[-M,M]$. By taking the absolute value and dividing my $M$, $\mathcal{D}_\text{era split}$ is then bounded on the interval $[0,1]$. At the low end, with a value of $\mathcal{D}_\text{era split} = 0$, half of the directions go in one way while half go in the other way, the highest level of disagreement. At the high end, with a value of $1$, all the directions go the same way, the most agreement. When choosing the best split during tree growth, higher values of $\mathcal{D}_\text{era split}$ are preferred.

\section{Theoretical Breakdown}
\label{sect:theory}

This section expounds three propositions which are illustrative of the motivations behind era splitting. Here the focus is on era splitting (eq. \ref{era_split_criterion}), and not directional era splitting (eq. \ref{eq:dir-era-split}). The former being more anchored in traditional theory, and connected with the original criterion. The later is new and experimental, with impurity tied to the direction of the predictions in each era. Nevertheless, the following proposition \ref{prop:reg} applies to both. This section contains all new, original content.

One problem with traditional splitting methods in the OOD setting is that the "optimal" split can be \textit{degenerate} in any or all of the eras (environments). The term, degenerate, indicates that the split doesn't result in improved impurity in the child nodes compared to the parent for any particular era of the data set. 

\begin{definition}[Degenerate Splits]
    A split is considered degenerate if $\mathcal{L}_{\text{split}}^j <= 0$, or is undefined, for any era $j \in M$ for $M$ eras of training data. Where $\mathcal{L}_{\text{split}}^j$ is defined in equation (\ref{eq:era-loss}). The expression is undefined if the split results in one completely empty child node.
\end{definition}

This is a split that never would have been if the tree was grown on that era of data independently. The left side of figure \ref{fig:example-degen-era-split} shows a simple example of a degenerate split with a toy data set of 4 data points. The original split criterion favors the split (red dashed line) which effectively separates era 1 data from era 2 data. In each era, this split does not induce any reduction in impurity. The era splitting criterion, on the other hand, favors a split (green dotted line) which reduces impurity in both eras simultaneously. Traditional splitting methods can induce splits that don't split the data in some particular environment at all.

\begin{proposition}
    \label{prop:prop1}
    The original splitting criterion defined in equation \ref{original_split_criterion} can result in degenerate splits.
\end{proposition}

\begin{proof}

Consider a data set consisting of two eras of data, two features and two data points, as defined in table \ref{tab:degen-data}, and appearing on the left plot of figure \ref{fig:example-degen-era-split}.

\begin{table}
    \centering
    \begin{tabular}{r|l}
        Feature 1 & 1, 2, 3, 4 \\
        Feature 2 & 1, 3, 2, 4 \\
        Era Identifier & 0, 0, 1, 1 \\
        Target & -1, -2, -3, -4 \\
        Prediction & 0, 0, 0, 0 \\
        Gradient & 1, 2, 3, 4 \\
    \end{tabular}
    \caption{Data for proposition \ref{prop:prop1}.}
    \label{tab:degen-data}
\end{table}

The optimal split score for the original split criterion defined in equation \ref{original_split_criterion} is equal to 2 and splits feature 1 between values 2 and 3. When computed independently over each era, this split results in an undefined score in both era 1 and era 2, since it does not split the data in either era. This split is degenerate. 

\end{proof}

This is one of the main conceptual motivations for the current research: it is possible that the original method misses, ignores, or misinterprets key environmental information. Era splitting addresses this problem in a very direct way: evaluate each split based on each environment separately. There is one configuration of era splitting which ensures the split points will \textbf{always} reduce impurity in every environment (era) of data simultaneously, never producing degenerate splits. This result is presented in proposition \ref{prop:no-degen}.

\begin{proposition}
\label{prop:no-degen}
The era splitting criterion, as defined in equation (\ref{era_split_criterion}), with Boltzmann alpha parameter approaching $-\infty$ will never induce a degenerate split.
\end{proposition}

\begin{proof}

The Boltzmann operator $\mathcal{B}_\alpha$ approaches the minimum operator as the value of $\alpha$ approaches infinity.

$$\mathcal{B}_\alpha \to \min\ \text{as}\ \alpha \to -\infty$$

In the context of era splitting (equation \ref{era_split_criterion}), the input to the operator are the split criterion scores from equation \ref{original_split_criterion} computed in each era independently. The minimum is the lowest score from any era. In order to induce a split, this score must be greater than zero, as per the rules of node splitting. The worst improvement will still be positive. Thus, the era splitting criterion with Boltzmann alpha parameter of $-\infty$ will only induce splits which reduce impurity in every era simultaneously.

\end{proof}

A simple but important upshot of era splitting is that it always produces \textbf{less impure} splits, when viewed from the traditional setting. If we evaluate all splits using the original splitting criterion, other splits induced by other splitting criteria will, by definition, be worse (result in more impurity in the child nodes when all the data is pooled together as in the traditional setting). This means that trees grown with other splitting criteria will fit the training data to a lesser degree than the original. Impurity will not have been reduced as much. In-sample performance will be poorer. The hope is that this will result in better generalization to OOS performance. This is exactly the concept of \textbf{regularization} in ML. The experiments go on to confirm that this is exactly what is happening.  

\begin{proposition}(Era Splitting as a Regularizer)
\label{prop:reg}

At each node, the era splitting (eq. \ref{era_split_criterion} ) and directional era splitting (eq. \ref{eq:dir-era-split}) criteria will always choose splits with a score \textbf{less than or equal to} (worse than) the original criterion (eq. \ref{original_split_criterion}), when measured with the original criterion.

\end{proposition}

\begin{proof}
    Assume the original criterion (eq. \ref{original_split_criterion}) is used to find the optimal split point at a node. This split attains the highest score by definition. It is impossible for the era splitting (or any other) criterion to induce a split with a better score than the optimal split point, when measured using the original criterion. 
\end{proof}

\section{Experimental Methods}
\label{section:methods}

Four experiments have been designed to empirically validate the newly proposed splitting criteria. The first experiment is an original design, called \textit{the shifted sine wave}, validating the era splitting models in a simple one-dimensional setting. The second and third experiments are binary classification problems from the OOD literature: the synthetic memorization data set presented in \cite{parascandolo2020learning} and the Camelyon17 (CAncer MEtastases in LYmph nOdes challeNge) data set from \cite{bandi2018detection}. These two experiments confirm the new splitting criteria are effective in known OOD settings. The forth is the target application of this research, the Numerai data set. 

The experiments follow a similar procedure for each. Random configurations are drawn from a grid of parameter values. The parameters include the number of boosting iterations, maximum number of leaves, learning rate, etc. All relevant parameters in Scikit-Learn's HistGradientBoostingRegressor \cite{scikit-learn} (the baseline model) are available. The number of random configurations (between 15 and 30) and the exact parameter ranges vary slightly from one data set to another, depending on training time and specifics of each data set. For the full list of parameters and example values, see table \ref{tab:params}. See the code for details, referenced at the end of this subsection. Each of the three experimental models are paired with each random configuration: \textit{original}, \textit{era splitting}, and \textit{directional era splitting}. In this way all three models are trained and evaluated over the same set of parameter values. In each case, the model is trained on the the training set and evaluated on the OOS test set. Evaluation statistics are recorded for each configuration. Accuracy is the key statistic for classification problems. The regressor is converted to a binary classifier by rounding the predictions to the nearest integer (0 or 1). For regression problems the mean-squared error (MSE) and Pearson correlation (Corr) are observed.

\begin{table}
    \centering
    \begin{tabular}{r|l}
        \textbf{Parameter} & \textbf{Example Range} \\
        Column Sample by Tree & 0, 0.1, 0.3, 0.5, 0.7, 0.9, 1 \\
        L2 Regularization & 0, 0.2, 0.4, 0.6, 0.8, 1 \\
        Learning Rate & 0.01, 0.05, 0.1, 0.5, 1.0 \\
        Max Number of Bins & 3, 4, 5, 7, 9 \\
        Max Depth & 2, 3, 4, 5, 7, 9, 15 \\
        Max Number of Leaves & 5, 7, 10, 16, 32 \\
        Min Child Samples & 1, 3, 5, 10, 20 \\
        Boosting Rounds & 5, 10, 20, 50, 100, 150 \\
        Split Type & Original, Era Split, Dir. Era Split \\
        Boltzmann Alpha & -2, -1, 0, 1, 2 \\
    \end{tabular}
    \caption{Caption}
    \label{tab:params}
\end{table}

\subsection{The Shifted Sine Wave Data Set}

The first experiment is a proof of concept for invariant learning in a simple regression setting. There is a 1-dimensional input feature which has an invariant aspect that is obscured by two types of randomness. One is a consistent Gaussian blur and the other is a random shift which is consistent inside each era but different across different eras of data. This part in particular simulates the distributional shift in the data generation process from one era to another.

The number of eras (environments) and the number of data points per era are pre-defined. The 1-dimensional input data is a random number on the interval $[0, 2 \pi ]$. The target variable is a function of the the input. The invariant part of the target is a sine wave plus standard Gaussian noise. The environmental distribution shift is realized through a vertical shift of the target, where the magnitude of the shift is random, but consistent throughout each era. In this way, each era of data is a blurry sine wave that is shifted by a random amount. Pooling many eras of data together into one data set forms a cloud of data where the invariant mechanism, the sine wave, fades away and is not visually apparent. The experimental data set consists of 8 eras of data, each era having 64 data points. 

\begin{figure}
    \centering
    \includegraphics[width=\linewidth]{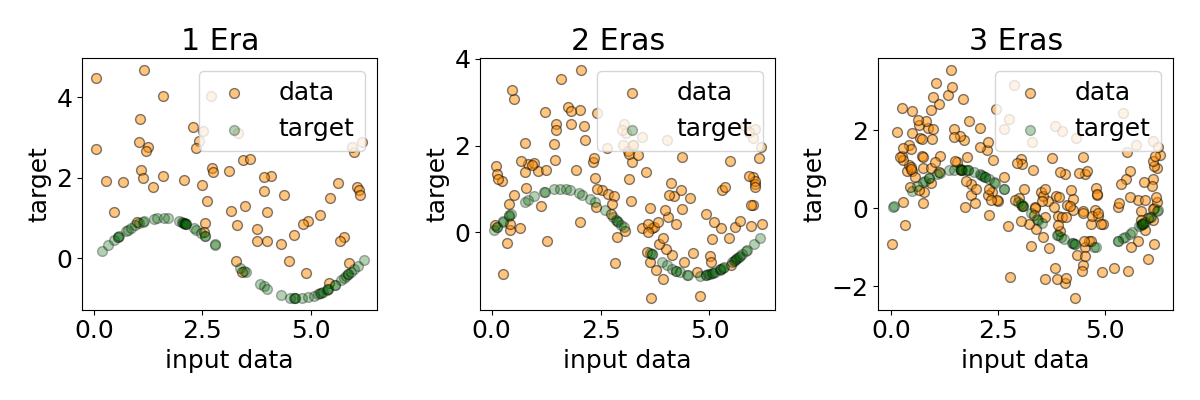}
    \caption[The Shifted Sine Wave Dataset]{A visual description of the data generation process for the shifted sine wave data set. Each era of training data starts with a sine wave (green), adds a random vertical shift and a random blur (Gaussian noise).}
    \label{fig:eras123}
\end{figure}

Evaluation metrics are computed against the test data, which is randomly generated by an additional era. 

\subsection{ Synthetic Memorization Data Set }
\label{sec:synthetic-desc}

The synthetic memorization data set was developed in  \cite{parascandolo2020learning} specifically to investigate an ML model's ability to learn invariant predictors in high-dimensional data in the presence of confounding spurious signals. It is a supervised binary classification task with target labels taking the values $\{0,1\}$. The data set is composed such that the first two input dimensions encode the invariant signal, which is an interlaced spiral shape, where one of the arms belongs to one class, and the other arm the other class. The remaining input dimensions are used to encode simple \textit{shortcuts}, or spurious easy-to-learn signals. These are simply two clusters of data points in high dimensions, where one cluster belongs to one class and the other the other class. The clusters are linearly separable, being easy to learn. The spurious aspect is realized through a mechanism where each era (environment) of data has a shift in the location of the two clusters. The shift happens in such a way that, when all the environments of the training data set are pooled together, the shortcut is still present. Refer to figure \ref{fig:synthetic-mem-dataset} or figure 5 from \cite{parascandolo2020learning} for a visualization of the data in four dimensions. When environments 1 \& 2 are pooled together, there is still an easily identifiable linear decision boundary present in the spurious signal feature dimensions.

\begin{figure}
    \centering
    \includegraphics[width=0.85\linewidth]{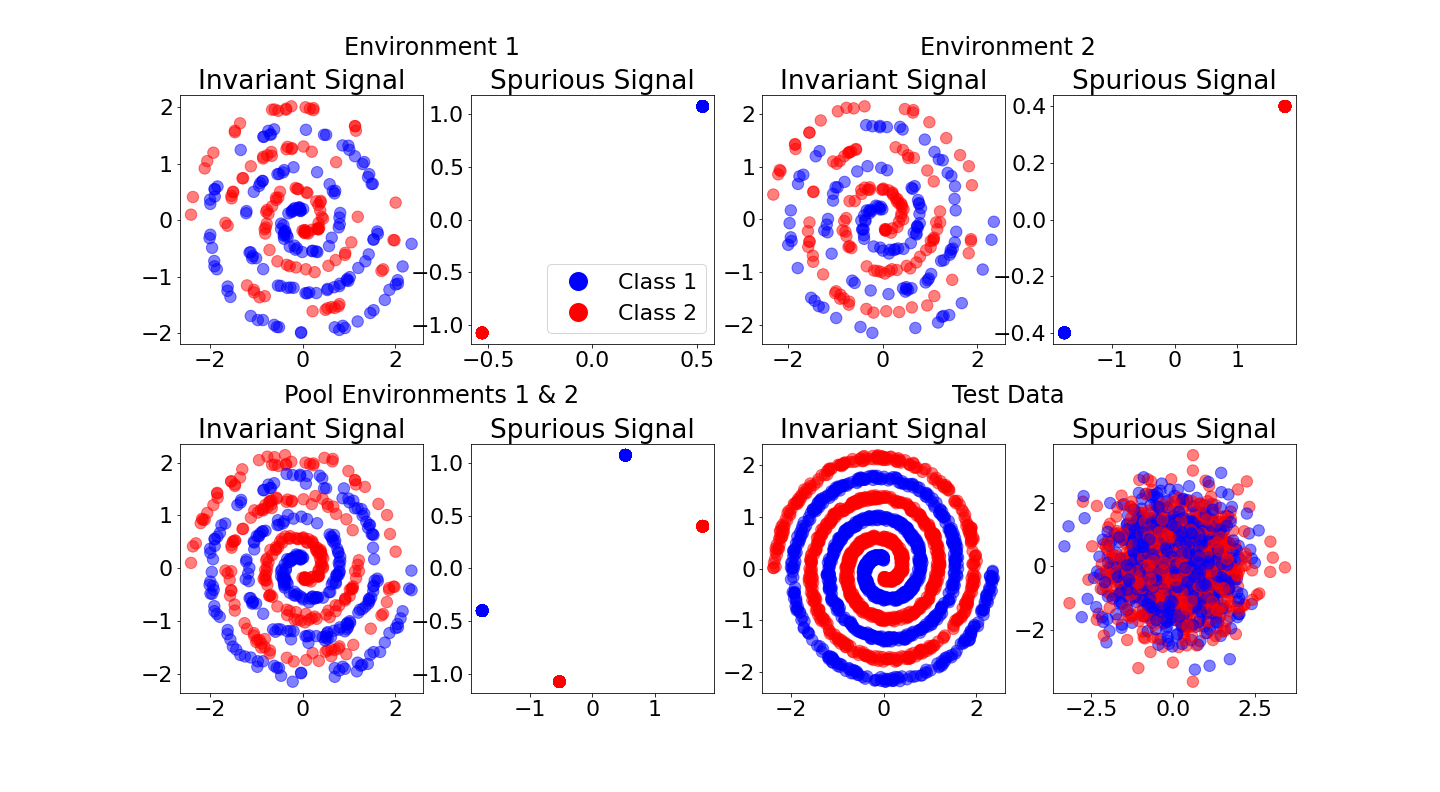}
    \caption[The Synthetic Memorization Data Set \cite{parascandolo2020learning}]{The synthetic memorization data set from \cite{parascandolo2020learning}. Each environment contains a random sample of the invariant spiral signal along the first two dimensions. The third and forth dimension contain a simple linearly-separable \textit{spurious} signal that changes with each environment. When pooling environments together, the pooled spurious signals still create a simple linear decision boundary. During OOS testing, the spurious signal is replaced by random noise. }
    \label{fig:synthetic-mem-dataset}
\end{figure}

During test time the spurious signal is replaced by random noise. Any model that bases its predictions solely on the spurious signal will be worthless when evaluated on the test data. If a model has actually learned the invariant, albeit harder-to-learn signal, it will still perform well on the test set. The authors of \cite{parascandolo2020learning} developed a technique called the AND-mask, which is an alteration to the naive gradient descent algorithm, incorporating the era-wise information. They were able to show that a naive MLP model simply learns the spurious signals in the data, and performs poorly on the test data. Meanwhile their AND-mask model was able to learn in such a way where it still performs well (98\% accuracy) on the test data.

The training data contains 12,288 data points, each having 18 dimensions, coming from 16 eras. The test set is made of 2,000 data points. In order to use a regression model as a classifier, the model is trained on the classification labels $\{0,1\}$. Predictions are then rounded to the nearest integers, which is either a $0$ or a $1$. The accuracy of the predictions can then be computed in the usual way according to 2-class classification problems.

\subsection{Camelyon17 Data Set}

The Camelyon17 data set \cite{bandi2018detection} is an example of a real-world dataset exhibiting distributional shifts. This is one of the \textit{Wilds} data sets that is distributed and supported by a team at Stanford \cite{wilds2021}. This data set aims at the automation of breast cancer detection in medical settings. The input data is a $96 \times 96$ pixel histopathological image and the target label is a binary indicator that is positive if the central $32 \times 32$ pixel region contains any tumor tissue. There are variations in data generation and processing technique from one hospital to another. The domain generalization problem is to learn to detect tumors from images from a number of hospitals, and have this detection ability transfer to new hospitals not in the training set. 

Accompanying each input data point, there is an integer identifier corresponding to the hospital. This identifier is used as the training era (environment). The raw size of the image has $96 \times 96 \times 3 = 27,648$ features, which was taxing the memory capabilities of the research computer. To reduce the memory overhead the length and width of the inputs are reduced by a factor of 1/3 using the Wilds API, reducing the number of features to a more manageable $3,072$. There are $302,436$ data points in the training set coming from $3$ hospitals. The validation set contains $34,904$ data points from one hospital and the test set contains $85,054$ data points from one hospital.  The Wilds data set provides established train, validation and test sets. Each model configuration is trained on the training set, and evaluated on the test set.

\subsection{Numerai Data Set}

Numerai is a San Francisco-based hedge fund which operates via a daily data science tournament in which the company distributes free, obfuscated data to participants, and those participants provide predictions back to Numerai. The predictions represent \textit{expected alpha}, which is a directional up/down, bull/bear style prediction. On any day, a tournament participant will need to produce predictions for 5,000 to 10,000 unique stocks. The predictions are then scored against the true realizations that transpire in the markets, which then become the basis for future target labels.

Numerai distributes a large data set designed for supervised learning. Each row of data represents one stock at one time (era). The rows consist of many input feature columns, and also several target (output) columns. This research uses the main scoring target, called \textit{Cyrus}. The Cyrus target values correspond to the optimal expected result over the subsequent 20 trading days, which is roughly equivalent to the stock's future returns in the market subject to company-level risk requirements. 

The raw time periods (eras) are numbered sequentially, starting at 1, and occur at weekly intervals. Each era represents one week of time. Currently there are over 1,060 eras of historical input and target pairs in the data set, and new eras are added each week, as the data becomes available. Each era contains more than 5,000 rows of data, on average, so the entire data set consists of over 5 million rows. The time horizon of the targets spans 4 weeks. Currently the full data set provides over 1,600 input features. In order to reduce training resources, a compact set of 244 features made popular by the tournament participant ShatteredX is used as our base feature set. It was shown with a baseline model that these features lead to out performance of the model which uses all the features. 

For the experimental data, the entire data set is partitioned into 5 folds over time. The first 4 folds are used for training, the the fold identifier is used as the era identifier, in place of the true era. The last fold is used for evaluation. This process is important and effectively reduces the number of eras in the data from over 1,000 to 5. The terminology can get confusing, so these larger eras are called \textit{training eras}. Each training eras spans about 4 years of time. This is a parameter choice that can reduce the complexity of the model. Reducing the number of training eras reduces the number of iterations in the algorithm, but also can lead to different performance capabilities. The current experiment did not explore fully the whole range of training era sizes, due to time constraints.

A notable grid search parameter is the number of estimators (boosting rounds), which varies between 50 and 1,000, to conserve training time. One "benchmark" model well-known to the community is the LightGBM model with 2,000 trees, maximum depth of 5, 32 leaves, learning rate of 0.01 and column-sample-by-tree parameter of 0.1. This model is trained separately in order to make a comparison of our results to a well-known model.

The main evaluation metric is the \textit{era-wise correlation}, which is called \textit{corr}. The era-wise correlations are the mean Pearson correlations of the predictions with the targets computed over each era of data independently. A secondary important metric is called \textit{corr Sharpe}, which is the mean era-wise correlation divided by the standard deviation of the era-wise scores.

\section{Results}
\label{results}

\begin{table}
    \centering
    \begin{tabular}{@{} *5l @{}}    \toprule
        \emph{Experiment} & \emph{Original} & \emph{Era Split} & \emph{Dir. Era Split} \\\midrule
        Sine Wave (MSE) & 1.79  & \textbf{1.71} & 1.75  \\ 
        Synth. Mem. (Acc) & 52 & 88 & \textbf{96}\\ 
        Camelyon17 (Acc) & 60 & 65 & \textbf{78}\\
        Numerai (Corr) & 0.0236 & 0.0203 & \textbf{0.0265}\\\bottomrule
         \hline
    \end{tabular}
    \caption{Best Test Scores by Model.}
    \label{tab:scores}
\end{table}

The results for all four experiments are summarized in table \ref{tab:scores} and figure \ref{fig:scores-generalization}. Table \ref{tab:scores} tabulates the best evaluation score over all the random configurations for each model split type. In each experiment the best model is either an era splitting model or a directional era splitting model. Figure \ref{fig:scores-generalization} displays box plots of the evaluation metrics for both training and test sets over all the model configurations, separated by split type. A consistent trend emerges. The original split models are characterized by the best in-sample performance and the worst OOS performance, ie. the largest generalization gap. The era split and directional era split models have decreased in-sample and increased OOS performance, a smaller generalization gap, consistently in every experiment. This supports the claim of proposition \ref{prop:reg}, that these new splitting criteria act as effective regularizers.

\begin{figure}
    \centering
    \includegraphics[width=.45\textwidth]{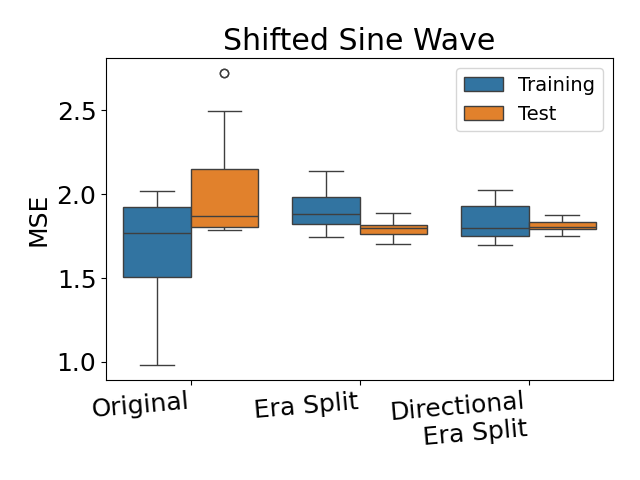}
    \includegraphics[width=.45\textwidth]{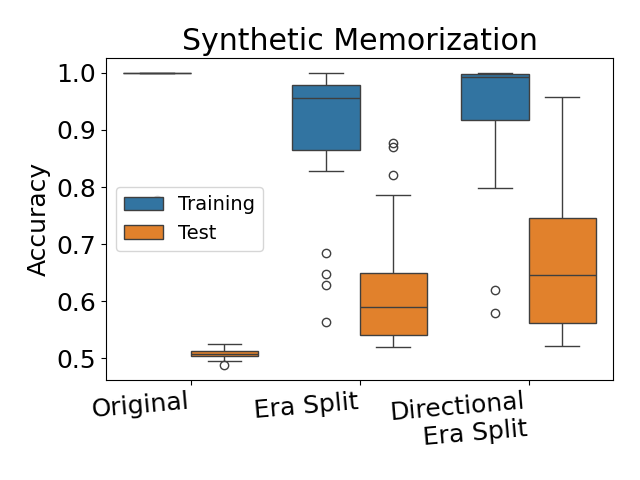}
    \includegraphics[width=.45\textwidth]{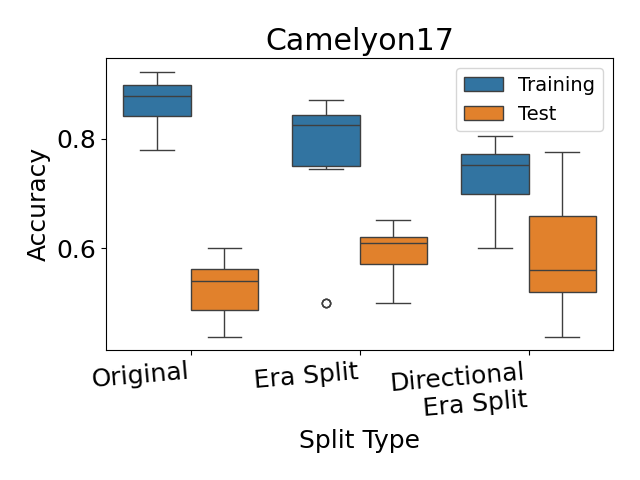}
    \includegraphics[width=.45\textwidth]{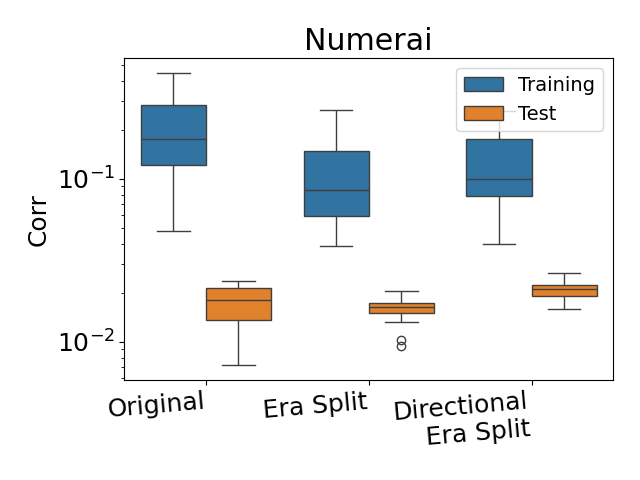}
    \caption[Box Plots of Evaluation Metrics for Era Splitting Experiments]{Box plots of the evaluation metrics for each experiment, for training and test sets.}
    \label{fig:scores-generalization}
\end{figure}

Figures \ref{figure:sine-wave-experiments} and \ref{fig:synthetic-mem-figure} visualize the decision surfaces as a function of input data values for the shifted sine wave and synthetic memorization data sets. The sine wave experiment is one dimensional and easy to plot. The original split criterion results in a much noisier decision boundary than both era splitting and directional era splitting for this configuration. The synthetic memorization data set is high-dimensional, with the invariant swirl signal encoded on the first two dimensions. Plotting the predicted values as a function of these first two dimensions reveals whether or not the model has learned to use this signal. The original split model has been unable to learn the swirl signal, as well the predicted values are very extreme, meaning the model is very confident in those predictions. The high in-sample and low OOS accuracy scores indicate that these models have only learned the spurious signals which disappear in the OOS test set. On the other hand, both the experimental models have clearly learned the swirl signal, with higher predicted values along one arm and lower values along the other.

\begin{figure}
    \centering
    \begin{subfigure}[b]{0.4\textwidth}
        \includegraphics[width=\textwidth]{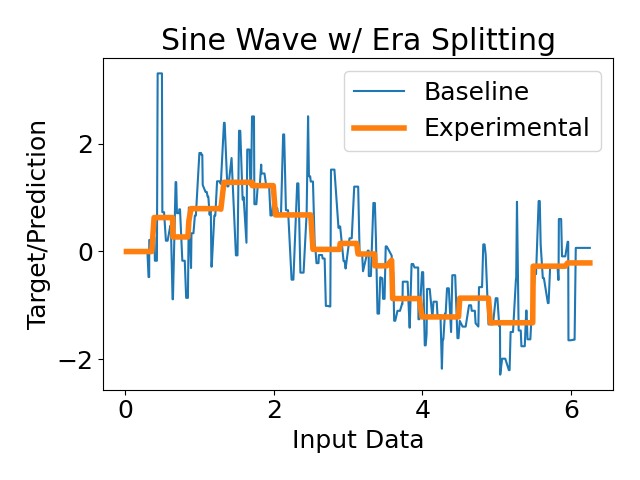}
        \caption{Era Splitting}
    \end{subfigure}
    \begin{subfigure}[b]{0.4\textwidth}
        \includegraphics[width=\textwidth]{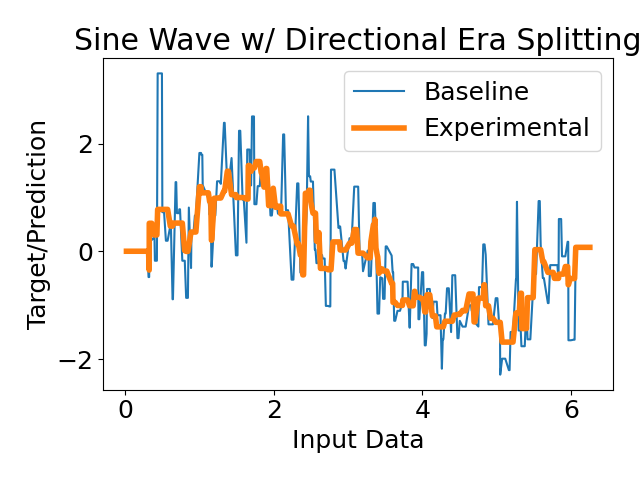}
        \caption{Directional Era Splitting}
    \end{subfigure}
    \caption[Visualization the Decision Boundary for the Sine Wave Data Experiment]{Visualization of the decision boundary of the baseline model vs. the experimental model. On the left era splitting, on the right, directional era splitting. Both baseline and experimental model are trained with identical base parameters: boosting rounds = 100, max depth = 10, learning rate=1, and all other parameters default.}
    \label{figure:sine-wave-experiments}
\end{figure}

Both new splitting criteria improve OOS results on the Camelyon17 data set, indicating that these new criteria can be useful in real-world health related applications outside of finance. 

The Numerai experiment shows that directional era splitting tends to consistently surpass the original splitting criterion, while the era splitting criterion did no surpass that of the original. In the direct comparison with the 2,000 tree benchmark model, the directional era splitting attained higher OOS corr with the target, 0.0271 compared too 0.0261 of the original. Results can be more nuanced than the bottom-line evaluation metric can show. In such a large parameter space, 15 to 20 configurations does not really explore a large portion of the space. To give more insight into the problem, there are a couple extra notebooks with experiments spanning more model parameters available in the online repository. In an additional experiment where the number of training eras is also varied from 3 to 8, several configurations of directional era splitting attained high correlation Sharpes, topping out at 1.26. That experiment can take 5 days or more to complete for 25 configurations. 

\section{Discussion}
\label{section:discussion}

The new splitting criteria presented in this research improve OOS performance and reduce in-sample performance of GBDT models, reducing the generalization gap in 4 concrete experiments. This supports the claimed roll as an effective regularizer. These experiments contain known and unknown distributional shifts within the training data and also between training and tests sets. Trees grown according to our new splitting criteria are more likely to generalize in these settings than trees grown according to the original splitting criterion.

Both of the new criteria are especially effective in the synthetic memorization problem and the Camelyon17 data set. The synthetic memorization contains a non-linear spiral signal as the invariant signal, while the true invariant signals present in the Camelyon17 data set are unknown. The positive results here showcase that there is a lot of information to lose if one does not take into account the domain (environmental) information.

The Numerai data set proved to be the most difficult to improve over the baseline, but the directional era splitting consistently did. It should be noted here that the era splitting criterion also tended to out-perform in preliminary experiments where a small grid was used, with a lower number of total trees. As the number of trees increased, the era splitting criterion itself appeared to become less effective. However there is still a lot left to explore here. While the era splitting correlation score ends up lower than the benchmark, other statistics can be much better, like correlation Sharpe and maximum draw down. The current Numerai experiment took 48 hours to train and evaluate 15 configurations. There are still many model configurations to experiment with. The number of training eras can also be varied and adapted, and this space is vast and relatively unexplored. An exploratory experiment showed some promise in this direction. In future experiments, the number of models will be expanded to several hundred, or even 1,000, which will take weeks or months to complete training. While out performance of individual models appears modest, the added variety from the the new split criteria can create the opportunity for more robust \textit{ensembles}, which are averaged groups of models.

This research has taken a large step in the development of invariant learning for GBDTs, but is still in its infancy. Directional era splitting appears to work the best in real-world data sets, even though it strays farther from traditional ideas in decision trees than the era splitting criterion. Instead of variance-based impurity measures, this uses a \textit{directional}-based measure, which is completely new to the field. This provides fresh ground for further theoretical interpretation and research.

\begin{figure}
    \centering
    \includegraphics[width=\textwidth]{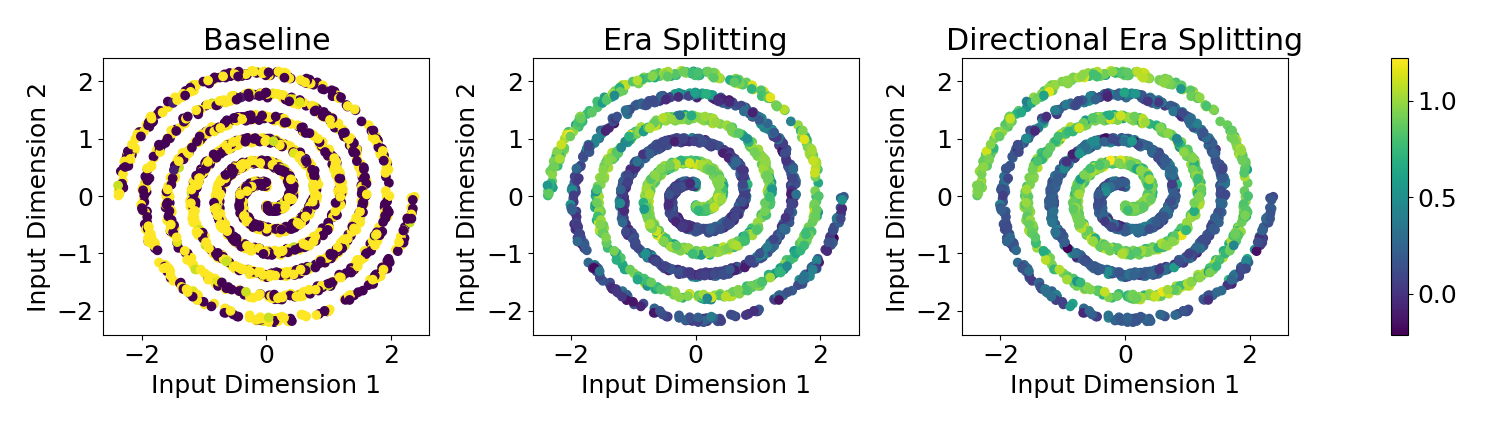}
    \caption[Visualization the Decision Boundary for the Synthetic Memorization Data Experiment]{ Comparison of the decision boundaries between the experimental models on the first two dimensions of the synthetic memorization data set. The color of the data point shows the predicted value based on that input. The original model was unable to learn the swirl signal while era splitting and directional era splitting models were able to. }
    \label{fig:synthetic-mem-figure}
\end{figure}

\subsection*{Time Complexity}

One of the biggest obstacles to the wide-spread adoption of this model is the added time complexity that is added by the era splitting routine. Indeed, one of the main issues with GBDTs in general is that, in order to grow a tree, each node must compute the split score (equation \ref{original_split_criterion}) for every potential split in the entire data set. The calculation must be made for every distinct value for every feature in the data. The histogram method \cite{Chen_2016} in modern software libraries has helped by reducing the number of unique values in the features down to a set parameter value. Setting this parameter to a low value, like 5, can greatly reduce training time. The \textit{colsample bytree} method also reduces the total number of computations needed by randomly sampling a subset of the features to use to build each tree, similar to the way bagging can reduce the number of rows.

In Era Splitting, the split gain must be computed not once but $M$ times per split, where $M$ is the number of training eras in our training data. For the Numerai data, $M$ can potentially take values up to 1,000. Run times for this implementation would take 1,000 times longer than the baseline model. This becomes prohibitive, as training times for larger models can easily take several days to complete training. In the experiment, the number of raw eras was reduced down to 5 training eras, which has a small time complexity and good performance. In other domains, such as the health domain characterizing the Camelyon17 data set, environments (such as hospitals) could be grouped together in logical ways, perhaps by geographic location. This, however, wasn't an issue for Camelyon17 data, which only came from a handful of hospitals. More progress could be made to understand the trade offs, limitations, and performance affects pertaining to the number of training eras in the data.

\section{Paper Code}

All the code for reproducing the experiments presented in this paper is available in the following two links. The first contains the forked version of Scikit-Learn\cite{scikit-learn} which implements the era splitting models, and the second contains all the notebooks to recreate the figures and experiments from this research.

\begin{itemize}
    \item \textbf{\url{github.com/jefferythewind/scikit-learn-erasplit}}
    \item \textbf{\url{github.com/jefferythewind/era-splitting-notebook-examples}}
\end{itemize}

\section*{Acknowledgments}

I would like to express my gratitude to Richard Craib, Michael Oliver, and the entire Numerai research team for their invaluable assistance and support throughout this project. This project was funded in part through an internship with Numerai from the summer of 2022 to 2023.



\bibliography{bib}

\end{document}